\documentclass[final]{IEEEtran}

\usepackage{amsmath,amssymb,amsthm,amsfonts}
\usepackage{graphicx,subcaption}
\usepackage{epstopdf}
\usepackage{algorithm,algorithmic}
\usepackage{cite}
\usepackage{tikz}
\usetikzlibrary{shapes,arrows,external}
\usetikzlibrary{positioning}
\usepackage{breqn}
\usepackage{multirow}

\newtheorem{thm}{Theorem}

\newcommand{\bE}{\mathbb{E}}
\newcommand{\bP}{\mathbb{P}}
\newcommand{\bR}{\mathbb{R}}

\newcommand{\bN}{\mathbb{N}}

\newcommand{\sF}{{\mathcal F}} 
\newcommand{\tnhat}{\hat{\theta}_n}

\newcommand{\argmin}{\operatornamewithlimits{arg\ min}}
\newcommand{\argmax}{\operatornamewithlimits{arg\ max}}

\newcommand{\ind}[1]{\mathbf{1}_{#1}} 

\newcommand{\set}[1]{\left\{#1\right\}}

\newcommand{\norm}[1]{\left\|#1\right\|}

\def\ie{{\em i.e.,~}}
\def\eg{{\em e.g.,~}}

\begin{document}

\title{Distance-Penalized Active Learning Using \\
Quantile Search}

\author{John~Lipor$^{1}$, Brandon Wong$^{2}$, Donald Scavia$^{3}$, Branko Kerkez$^{2}$, and Laura Balzano$^{1}$ \\
    $^{1}$Department of Electrical and Computer Engineering, \\
    $^{2}$Department of Civil and Environmental Engineering, 
    $^{3}$School of Natural Resources and Environment \\
    University of Michigan, Ann Arbor \\
    \{lipor,bpwong,scavia,bkerkez,girasole\}@umich.edu
    \thanks{This work was supported by NSF ECCS 1342121, F031543-071159-Graduate Research Fellowship Program, and the MCubed Program at the University of Michigan.}
}

\maketitle

\begin{abstract}
Adaptive sampling theory has shown that, with proper assumptions on the signal class, algorithms exist to reconstruct a signal in $\bR^d$ with an optimal number of samples. We generalize this problem to the case of spatial signals, where the sampling cost is a function of both the 
number of samples taken \textit{and} the distance traveled during estimation. This is motivated by our work studying regions of low oxygen concentration in the Great Lakes. 
    We show that for one-dimensional threshold classifiers, a tradeoff between the number of samples taken and distance traveled can be achieved using a generalization of binary search, which we refer to as \textit{quantile search}. 
We characterize both the estimation error after a fixed number of samples and the distance traveled in the noiseless case, as well as the estimation error in the case of noisy measurements.
We illustrate our results in both simulations and experiments and show that our method outperforms existing algorithms in the majority of practical scenarios.
\end{abstract}


\section{Introduction}
\label{sec:introduction}

Intelligently sampling signals of interest has been a fundamental topic in the signal processing community for many years, the most recent advances in this area being compressed sensing \cite{donoho2006compressed} and active learning \cite{settles2012active}. In these and other scenarios, the goal is typically to recover a signal from a given class (\eg bandlimited signals or the Bayes decision boundary for $0/1$ signals) using as few samples as possible. However, in the modeling of
spatial phenomena, such as oxygen concentration
in lakes, the sampling cost is a function of both the number of samples required \textit{and} the cost to travel to the sample locations. 
Therefore, the design of provably efficient algorithms to detect spatial phenomena is an important open problem and is the topic of this paper.

Consider our motivating problem, in which we wish to estimate the boundary of a hypoxic region (\ie a region of oxygen concentration below 2.0 ppm \cite{glerl2005hypoxia}) in the central basin of Lake Erie using an autonomous watercraft with a speed ranging from 0.5-4 m/s. Fig. 1 shows an interpolated estimate of the oxygen concentration based on a small number of samples taken throughout the lake, where the hypoxic region is in blue/purple. Oxygen concentration is a
strong indicator of the health of the Great Lakes \cite{glerl2005hypoxia} and the spatial extent of such regions is a topic of interest for researchers in the field \cite{beletsky2006modeling,scavia2014assessing}. 
We assume the hypoxic region is connected with a smooth boundary and
that the boundary remains relatively stationary over the course of a few days.
The problem of estimating the boundary can then be viewed as a binary classification problem, in which spatial points receive a label 0 if they are hypoxic and 1 otherwise, and the desired spatial extent corresponds to the Bayes decision boundary. While the application of optimal active learning algorithms such as \cite{singh2006active,castro2008minimax} minimizes the number of samples required to estimate the boundary, the distance traveled in the estimation procedure is prohibitive,
since these algorithms require a coarse sampling of the entire feature space (in this case, the entire central basin of Lake Erie).

\begin{figure}[t]
    \centering
    \includegraphics[width=2.5in]{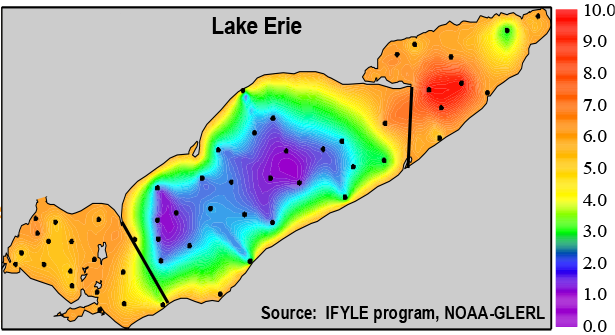}
    \caption{Dissolved oxygen concentrations in Lake Erie. Points represent sample locations and solid black lines delineate the central basin. \small{Source: http://www.glerl.noaa.gov/res/waterQuality/}}
    \label{fig:erie}
\end{figure}

\begin{figure*}[t]
    \centering
    \includegraphics[width=0.9\linewidth]{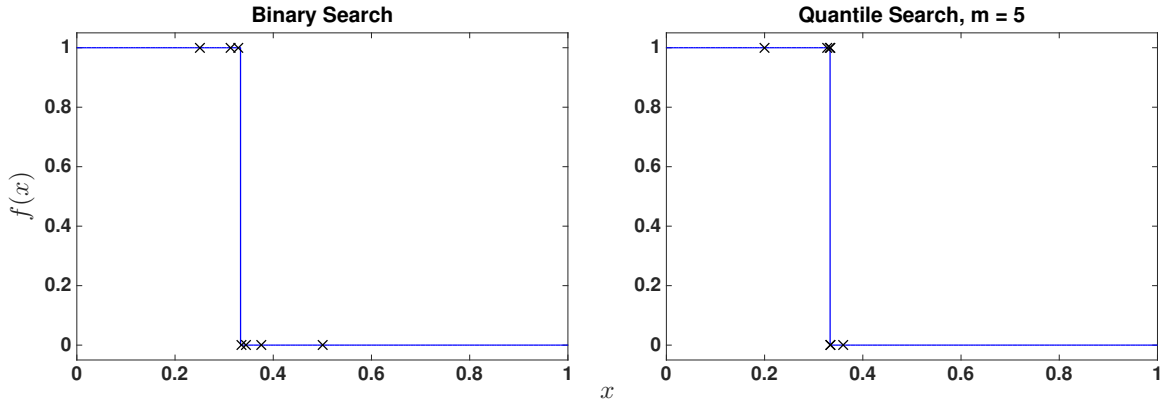}
    \caption{Example step function with $\theta = 1/3$ with corresponding measurements (marked by an x) taken using binary search (left) and quantile search with $m = 5$ (right).}
    \label{fig:step}
\end{figure*}

In this paper, we present an active learning algorithm called \textit{quantile search} that achieves a tradeoff between the number of measurements and distance traveled to estimate the change point of a one-dimensional step function. At its two extremes, quantile search minimizes either the number of samples or the distance traveled to estimate the decision boundary, with a tradeoff achieved by varying a search parameter. We derive the expected number of samples required and distance traveled
in the noiseless case and bound the number of samples required in the case of noisy measurements. We also show how a series of one-dimensional estimates can be used to estimate the two-dimensional boundary of interest. This paper is an extension of our previous work \cite{lipor2015quantile}. Our contributions beyond \cite{lipor2015quantile} are as follows. We provide detailed proofs of our theoretical results in the supplemental material \cite{lipor2016supplemental}.
We present a novel generalization in the case of noisy measurements that, unlike our previous work, is equivalent to the noiseless case when the probability of measurement error is zero. We also provide two algorithmic improvements for the problem of interest and show in simulations that these
greatly reduce the required sampling time. Our simulations are more realistic, including real bathymetry data from Lake Erie provided by the National Oceanic and Atmospheric Administration \cite{noaa2015bathymetry}.
We also compare the performance of our algorithm to a version of \textit{proactive learning} \cite{donmez2008proactive}.
Finally, we include results of our experiments performed on Third Sister Lake in Ann Arbor, MI with an autonomous watercraft controlled using a cloud-based architecture.

\section{Problem Formulation \& Related Work}
\label{sec:problem}

Determining the spatial extent of the hypoxic region shown in Fig. \ref{fig:erie} can be interpreted as learning a two-dimensional Bayes decision boundary. Following \cite{castro2008active}, we split our two-dimensional problem into several one-dimensional intervals, a process that is described further in Section \ref{sec:simulations}. In each interval, we must find a threshold beyond which the lake is hypoxic. Define the step function class 
\begin{equation*}
    \sF = \{f: [0,1] \rightarrow \bR | f(x) = \ind{[0,\theta)}(x)\}
\end{equation*}
where $\theta \in [0,1]$ is the change point and $\ind{S}(x)$ denotes the indicator function which is 1 on the set $S$ and 0 elsewhere. An example function belonging to $\sF$ with $\theta = 1/3$ is shown in Fig. \ref{fig:step}. In contrast to the standard active learning scenario, our goal is to estimate $\theta$ while minimizing the total time required for sampling, a function of both the number of samples taken \textit{and} the distance traveled. Denote the observations $\{Y_{n}\}_{n = 1}^{N} \in \set{0,1}^{N}$ as samples of an unknown function $f_{\theta} \in \sF$
taken at sample locations on the unit interval $\{X_{n}\}_{n = 1}^{N}$. With probability $p$, $0 \leq p < 1/2$, we observe an erroneous measurement. Thus
\begin{equation*}
    Y_{n} = \begin{cases}
        f_{\theta}(X_{n}) & \text{with probability $1 - p$} \\
        1 - f_{\theta}(X_{n}) & \text{with probability $p$}
    \end{cases} = f(X_{n}) \oplus U_{n},
\end{equation*}
where $\oplus$ denotes summation modulo 2, and $U_{n} \in \set{0,1}$ are Bernoulli random variables with parameter $p$. While other noise scenarios are common, here we assume the $U_{n}$ are independent and identically distributed and independent of $\set{X_{n}}$. This noise scenario is of interest as the motivating data (oxygen concentration) is a thresholded value in $\set{0,1}$, where Gaussian noise results in improper thresholding of the measurements. The extension to nonuniform noise (\eg a
Tsybakov-like noise condition as studied in \cite{castro2008minimax}) remains as a topic for future work.

\subsection{Related Work}

A number of active learning algorithms designed to estimate $\theta$ exist; however, these algorithms typically assume the sampling cost is due only to the measurements themselves. Most similar to our algorithm is the method of binary bisection and its extensions \cite{horstein1963sequential,burnashev1974interval,karp2007noisy,or2008bayesian,castro2008active,castro2008minimax,waeber2013bisection}. In the noiseless case, binary bisection estimates the change point of a step
function on the unit interval by successively halving the space of potential classifiers, termed the \textit{hypothesis space}. An example of this search procedure is shown in the left-hand plot of Fig. \ref{fig:step}. A noise-tolerant version of this algorithm was first presented in
\cite{horstein1963sequential}, where measurements are flipped with known probability $p$. A discretized version of this algorithm was analyzed in \cite{burnashev1974interval} and shown to be minimax optimal in \cite{castro2008minimax} under the Tsybakov noise condition. Further, the authors of \cite{castro2008minimax} use the discretized algorithm to show that a series of one-dimensional threshold estimates can be used to estimate functions belonging to the boundary fragment
class in $d$ dimensions at a minimax optimal rate. The original algorithm presented in \cite{horstein1963sequential} was recently shown to converge at a geometric rate in \cite{waeber2013bisection}. Binary bisection has also been used to obtain optimal rates in optimization \cite{ramdas2013algorithmic} and in the noisy 20 questions problem \cite{tsiligkaridis2016asynchronous}. In \cite{castro2008active}, the authors give a spatial sampling
problem as motivation for the probabilistic binary search (PBS) algorithm. However, a simple analysis shows that in the noiseless case, to estimate the threshold of a step function on the unit interval, binary search travels the entire unit interval. Hence, while the worst-case number of samples required is minimized, the total distance traveled is the worst possible. In the motivating problem given above, the central basin of Lake Erie has a width of roughly 80 km, making this approach
prohibitive.

More sophisticated active learning algorithms have been widely studied, achieving optimal rates for piecewise constant functions in \cite{castro2008minimax} and for the linear support vector machine in \cite{wang2016noise}. In both cases, the algorithm begins by uniformly sampling the entire feature space. Again considering the problem of interest, the central basin of Lake Erie has an area of approximately 14,000 km$^{2}$, making this approach infeasible. In contrast, the algorithm studied in \cite{castro2008minimax} was used in
\cite{singh2006active} to measure the hydrodynamics of Lake Wingra in Madison, WI, which has an area of 1.3 km$^{2}$.

Nonuniform sampling costs are studied in \cite{liu2009spatially,demir2012cost,guillory2009average,donmez2008proactive}. In \cite{liu2009spatially}, the authors use the uncertainty sampling heuristic to determine the most informative points and penalize for spatial costs using the traveling salesman problem with profits. The work of \cite{demir2012cost} uses both uncertainty and diversity to select points and also penalizes for arbitrary costs. In both cases, the algorithm proceeds
in batches, \ie by iteratively requesting a set of labels and retraining the classifier. This approach suffers the same pitfalls as \cite{singh2006active} in that the algorithm can require traversing the entire feature space multiple times. Further, neither algorithm is accompanied by theoretical guarantees. In \cite{guillory2009average}, the authors present and analyze a greedy algorithm for active learning with nonuniform costs. However, as in our case the cost associated with
each point is the distance from the previous point, the costs in question are both nonuniform and dynamic. A somewhat similar algorithm, known as \textit{proactive learning}, is presented in \cite{donmez2008proactive}, where the proposed strategy chooses at each round the point maximizing the difference between or ratio of informativeness and cost to label the point. In Section \ref{sec:simulations}, we compare with this algorithm using mutual information as our metric for informativeness.

The problem of sampling spatial phenomena using mobile robots has been studied in signal processing and robotics literature as well. In \cite{unnikrishnan2013sampling,unnikrishnan2013high}, the authors study the case where the sampling cost is near-zero and show that equispaced parallel lines result in the minimum distance required to reconstruct a variety of practical signals. Mutual information is also used as a metric for informativeness
in \cite{singh2009efficient}, where the authors impose a Gaussian process model to perform path planning for robots used to track a variety of spatial phenomena. A greedy algorithm is presented with theoretical guarantees based on submodularity \cite{krause2008near}. However, the model imposed is not appropriate for determining the boundary of a region of interest. The recent work of \cite{bayram2016gathering} considers the measurement
cost and travel time to estimate the location of point targets using mobile robots but does not easily extend to the case of estimating the boundary of a region of interest. The algorithm in \cite{rahimi2005adaptive} is similar to the one described in
\cite{donmez2008proactive}, with the main difference being that in early stages the algorithm emphasizes regularity of samples (\ie encourages early samples to be taken uniformly throughout the feature space). 

\section{Quantile Search}
\label{sec:QS}


In this section, we present our algorithm \textit{quantile search}, an extension of binary search and ideas in \cite{burnashev1974interval,castro2008active} to penalize both the sample complexity and distance traveled during the estimation procedure.
The basic idea behind this algorithm is as follows. We wish to find a tradeoff between the number of samples required and the total distance traveled to achieve a given
estimation error for the change point of a step function on the unit interval. As we know, binary bisection minimizes the number of required samples. On the other hand, continuous spatial sampling
minimizes the required distance to estimate the threshold. Binary search bisects the feasible interval (hypothesis space) at each step. In contrast, one can think of continuous sampling as dividing the feasible interval into infinitesimal subintervals at each step. With this in mind, a tradeoff becomes clear: one can divide the feasible interval into subintervals of size $1/m$, where $m$ is a real number between 2 and $\infty$. Intuition would tell us that increasing $m$ would increase the number of
samples required but decrease the distance traveled in sampling. In what follows, we show that this intuition is correct in both the noise-free and noisy cases, resulting in two novel search algorithms.

\subsection{Deterministic Quantile Search}
\label{sec:DQS}

We first describe and analyze quantile search in the noise-free case ($p = 0$), here referred to as deterministic quantile search (DQS). To estimate the change point of a step function, deterministic binary bisection travels either forward or backward (depending on the measurement) a fraction 1/2 into the feasible interval. In contrast, the DQS algorithm presented here travels $1/m$ forward or backward, where
$m \in [2,\infty)$. While the DQS measurements for $m > 2$ are less informative than in binary bisection, we expect that the distance traveled during the estimation procedure will be reduced, since we can pass the change point by a fraction at most $1/m$. The search procedure for the case of $m = 5$ is shown in the right-hand plot of Fig. \ref{fig:step}. Note that in contrast to binary search, quantile search does not overshoot the change point $\theta = 1/3$ by a significant amount. A formal description of the procedure is given in Algorithm \ref{alg:DQS}. In the following subsections, we analyze the expected sample complexity and distance traveled for the algorithm and show the required
number of samples increases monotonically with $m$, and the distance traveled decreases monotonically with $m$, indicating that the desired tradeoff is achieved.

\begin{algorithm}[t]
    \caption{Deterministic Quantile Search (DQS)}
    \label{alg:DQS}
    \begin{algorithmic}[1]
        \STATE \textbf{Input:} search parameter $m$, stopping error $\varepsilon$
        \STATE \textbf{Initialize:} $X_{0} \gets 0$, $Y_{0} \gets 1$, $n \gets 1$, $a \gets 0$, $b \gets 1$
        \WHILE{$b-a > 2\varepsilon$}
        \IF{$Y_{n-1} = 1$}
        \STATE $X_{n} \gets X_{n-1} + \frac{1}{m}(b-a)$
        \ELSE
        \STATE $X_{n} \gets X_{n-1} - \frac{1}{m}(b-a)$
        \ENDIF
        \STATE $Y_{n} \gets f(X_{n})$
        \STATE $a = \max \set{X_{i}: Y_{i} = 1, i \leq n}$
        \STATE $b = \min \set{X_{i}: Y_{i} = 0, i \leq n}$
        \STATE $\hat{\theta}_{n} \gets \frac{a + b}{2}$
        \ENDWHILE
    \end{algorithmic}
\end{algorithm}

\subsubsection{Convergence of Estimation Error}

We analyze the expected error after a fixed number of samples for the DQS algorithm. The main result and a sketch of the proof are provided here. An expanded proof can be found in the supplemental material \cite{lipor2016supplemental}.

\begin{thm}
    \label{thm:dqs_rate}
    Consider a deterministic quantile search with parameter $m$ and let $\rho = \frac{m-1}{m}$. Begin with a uniform prior on $\theta$. The expected estimation error after $n$ measurements is then
    \begin{equation}
        \label{eq:dqs_err}
        \bE [|\tnhat - \theta|] = \frac{1}{4} \left[\rho^{2} + (1-\rho)^{2}\right]^{n}.
    \end{equation}
\end{thm}
\begin{proof}
    (Sketch; see complete proof in \cite{lipor2016supplemental}) The proof proceeds from the law of total expectation. Let $Z_{n} = |\tnhat - \theta|$. The first measurement is taken at $1/m$, and hence the expected error can be calculated when $\theta \leq 1/m$ and $\theta > 1/m$. 
    \begin{eqnarray*}
        \bE[Z_{1}] &=& \bE\left[Z_{1} | \theta \leq \frac{1}{m}\right] \bP\left(\theta \leq \frac{1}{m}\right) + \\
        && \;\; \bE\left[Z_{1} | \theta > \frac{1}{m}\right] \bP\left(\theta > \frac{1}{m}\right) \\
        &=& \frac{1}{4} \left[ (1-\rho)^{2} + \rho^{2} \right].
    \end{eqnarray*}
    Similarly, after the second measurement is taken, there are four intervals, two which partition the interval $[0,1/m]$, and two which partition $(1/m,1]$. These result in four monomials of degree 4, one of which is $(1-\rho)^{4}$, one which is $\rho^{4}$, and two which are $(1-\rho)^{2}\rho^{2}$. The basic idea is that each ``parent'' interval integrates to $(1-\rho)^{i}\rho^{j}$ and in the next step gives birth to two ``child'' intervals, one evaluating to $(1-\rho)^{i+1}\rho^{j}$ and the other
    $(1-\rho)^{i}\rho^{j+1}$. The proof of the theorem then follows by induction.
\end{proof}

Consider the above result when $m = 2$. In this case, the error becomes $\bE[|\tnhat - \theta|] = 2^{-(n+2)}$. Comparing to the worst case, we see that the average case sample complexity is exactly one sample better than the worst case, matching the well-known theory of binary search. In Section~\ref{sec:simulations} we confirm this result through simulation.

\subsubsection{Distance Traveled}

Next, we analyze the expected distance traveled by the DQS algorithm in order to converge to the true $\theta$. The proof is similar to that of the previous theorem in that it follows by the law of total expectation. After each sample, we analyze the expected distance given that the true $\theta$ lies in a given interval. The result and a proof sketch are given below, with the full proof included in the supplemental material \cite{lipor2016supplemental}.

\begin{thm}
    \label{thm:dqs_dist}
    Let $D$ denote a random variable representing the distance traveled during a deterministic quantile search with parameter $m$. Begin with a uniform prior on $\theta$. Then
    \begin{equation}
        \label{eq:dqs_dist}
        \bE[D] = \frac{m}{2m-2}.
    \end{equation}
\end{thm}

\begin{proof}
    (Sketch, see full proof in \cite{lipor2016supplemental}) We first consider the expected distance traveled before the algorithm reaches a point $x_{1} > \theta$. Let $D_{1}$ be a random variable denoting this distance. Once the algorithm passes this point, it moves in the reverse direction until reaching $x_{2} < \theta$, moving a distance $D_{2}$. This process repeats until convergence. Let $D_{n}$ be a random variable denoting the distance required to move to the right of $\theta$ for the $\left\lceil \frac{n}{2} \right\rceil$th time when $n$ is odd, and to the left of $\theta$ for the $\frac{n}{2}$th time when $n$ is even. In this case, we have that
    \begin{equation}
    \label{eq:dqs_tot_dist1}
        \bE[D] = \sum_{n = 1}^{\infty} \bE[D_{n}].
    \end{equation}
    \sloppypar{First, we would like to find $\bE[D_{1}]$. Let $A_{i}$ denote the interval $\left[\frac{1}{m}\sum_{p = 0}^{i-1}\left( \frac{m-1}{m} \right)^{p}, \frac{1}{m}\sum_{p = 0}^{i}\left( \frac{m-1}{m} \right)^{p}\right)$, where $A_{0} = \left[0,\frac{1}{m}\right)$, so that the $A_{i}$'s form a partition of the unit interval whose endpoints are possible values of the sample locations $X_{j}$.} Now note that
    \begin{equation*}
        \bE[D_{1}] = \sum_{i = 0}^{\infty} \bE[D_{1} | \theta \in A_{i}] \bP(\theta \in A_{i}).
        \label{eq:LTE}
    \end{equation*}
    Then since we assume $\theta$ is distributed uniformly over the unit interval, 
    \begin{eqnarray*}
        \bP(\theta \in A_{i}) &=&  \frac{1}{m}\sum_{p = 0}^{i}\left( \frac{m-1}{m} \right)^{p} - \frac{1}{m}\sum_{p = 0}^{i-1}\left( \frac{m-1}{m} \right)^{p} \\
        &=& \frac{1}{m}\left(\frac{m-1}{m}\right)^{i}.
        \label{eq:partition}
    \end{eqnarray*}
    Next, note that 
    \begin{eqnarray*}
        \bE[D_{1}|\theta \in A_{i}] &=&  \frac{1}{m}\sum_{p = 0}^{i} \left(\frac{m-1}{m}\right)^{p} \\
        &=& 1 - \left( \frac{m-1}{m} \right)^{i+1}.
        \label{eq:cond expec}
    \end{eqnarray*}
    Thus we have
    \begin{eqnarray*}
        \bE[D_{1}] &=& \sum_{i = 0}^{\infty} \bE[D_{1}|\theta \in A_{i}] \bP(\theta \in A_{i}) \\
        &=& \sum_{i = 0}^{\infty} \left[ 1 - \left(\frac{m-1}{m}\right)^{i+1} \right] \left[ \frac{1}{m} \left(\frac{m-1}{m}\right)^{i} \right] \\
        &=& \frac{m}{2m-1}.
    \end{eqnarray*}
    The proof proceeds by rewriting the above in terms of $\rho = (m-1)/m$ and then calculating $\bE[D_{n}]$.  This is done by dividing each $A_{i}$ into subintervals which form partitions of $A_{i}$. By induction we get 
    \begin{equation}
    \label{eq:EDn}
    \bE[D_{n}] = \frac{m}{(2m-1)^{n}} \;,
    \end{equation} and the result then follows from the infinite sum of \eqref{eq:dqs_tot_dist1}.
\end{proof}

\subsubsection{Sampling Time}

Using the above results, we wish to find the optimal tradeoff for a given set of sampling parameters. Let $\gamma$ be the time required to take one sample and $\eta$ be the time required to travel one unit of distance. The total sampling time $T$ is then
\begin{equation}
    \label{eq:time}
    T = \gamma N + \eta D,
\end{equation}
where $N$ denotes the number of samples required. Given a fixed sampling time and desired error, \eqref{eq:time} can be used to estimate the sample budget $N$. However, this approach differs from our goal of minimizing the total sampling time. Alternatively, the average value of $N$ can be estimated numerically and used to optimize the expected value of $T$.
We show examples of this approach in Section~\ref{sec:simulations} for both the deterministic and probabilistic versions of quantile search.

As a final note, one may wonder about the relation to what is known as m-ary search \cite{schlegel2009kary}. In contrast to quantile search, m-ary search is tree-based. To make the difference clear, consider the an example with $\theta = 3/8$ and let $m = 4$. In this case, both algorithms take their first sample at $X = 1/4$. However, after measuring $Y = 1$, quantile search takes its second measurement at $X = 7/16$, while m-ary search proceeds to $X = 1/2$. One may then expect that both
algorithms would achieve the desired tradeoff, with m-ary search using fewer samples and more distance for the same value of $m$. 
We focus on quantile search for two reasons. First, quantile search does not require $m$ to be an integer and therefore gives more flexibility in the resulting tradeoff. Second, quantile search as described is the natural generalization of PBS and lends itself to the analysis of \cite{burnashev1974interval,castro2008active} in the case where the measurements are
noisy. A comparison to noisy m-ary search is a topic for future work.

\subsection{Probabilistic Quantile Search}
\label{sec:PQS}

In this section, we extend the idea behind Section~\ref{sec:DQS} to the case where measurements may be noisy (\ie $p \geq 0$). In \cite{burnashev1974interval}, the authors present an algorithm referred to in the literature as \textit{probabilistic binary search} (PBS). The basic idea behind this algorithm is to perform Bayesian updating in order to maintain a posterior distribution on $\theta$ given the measurements and locations. Rather than bisecting the interval at each step, the algorithm bisects the posterior distribution. This process is then iterated until convergence and has been shown to
achieve optimal sample complexity throughout the literature \cite{castro2008minimax,or2008bayesian}. We now extend this idea using the quantile methodology of the previous section, resulting in what we term \textit{probabilistic quantile search} (PQS).

The idea behind PQS is straightforward. Starting with a uniform prior, the first sample is taken at $X_{1} = 1/m$. The posterior density $\pi_{n}(x)$ is then updated as described below, and $\tnhat$ is chosen as the median of this distribution. The algorithm proceeds by taking samples $X_{n}$ such that
\begin{equation*}
    \int_{0}^{X_{n+1}} \pi_{n}(x) dx = \frac{1}{m}.
\end{equation*}
For $m = 2$, the above denotes the median of the posterior distribution and reduces to PBS, while in general this denotes sampling at the $m$-quantile of the posterior. A formal description is given in Algorithm~\ref{alg:PQS}. 

\begin{algorithm}[t]
    \caption{Probabilistic Quantile Search (PQS)}
    \label{alg:PQS}
    \begin{algorithmic}[1]
        \STATE \textbf{Input:} search parameter $m$, probability of error $p$
        \STATE \textbf{Initialize:} $\pi_{0}(x) = 1$ for all $x \in [0,1]$, $n \gets 0$
        \WHILE{not converged}
        \STATE choose $X_{n+1}$ such that $\int_{0}^{X_{n+1}} \pi_{n}(x) dx = \frac{1}{m}$
        \STATE $Y_{n+1} \gets f(X_{n+1}) \oplus U_{n+1}$, where $U_{n+1} \sim$ Ber($p$)
        \IF{$Y_{n+1} = 0$}
        \STATE \begin{equation*}
            \pi_{n+1}(x) = \begin{cases}
                (1-p)\left(\frac{m}{1 + (m-2)p}\right) \pi_{n}(x), & x \leq X_{n+1} \\
                p\left(\frac{m}{1 + (m-2)p}\right) \pi_{n}(x), & x > X_{n+1}
                \end{cases}
            \end{equation*}
        \ELSE
        \STATE \begin{equation*}
            \pi_{n+1}(x) = \begin{cases}
                p\left(\frac{m}{1 + (m-2)p}\right) \pi_{n}(x), & x \leq X_{n+1} \\
                (1-p)\left(\frac{m}{1 + (m-2)p}\right) \pi_{n}(x), & x > X_{n+1}
                \end{cases}
            \end{equation*}
        \ENDIF
        \STATE $n \gets n + 1$
        \ENDWHILE
        \STATE estimate $\tnhat$ such that $\int_{0}^{\tnhat} \pi_{n}(x) = 1/2$
    \end{algorithmic}
\end{algorithm}

We derived the update for PQS in our previous work \cite{lipor2015quantile}, and it can be seen in steps 7 and 9 of Algorithm \ref{alg:PQS}. Here we derive a more general version of the update that will be referred to in Section~\ref{sec:proactive}. 
Begin with the first sample. We have $\pi_0(x) = 1$ for all $x$ and wish to find $\pi_1(x)$. Let $f_1(x|X_1, Y_1)$ be the conditional density of $\theta$ given $X_1, Y_1$. Applying Bayes rule, the posterior becomes:
\begin{equation*}
    f_1(x | X_1, Y_1)  = \frac{\bP(X_1, Y_1 | \theta = x) \pi_0(x)}{\bP(X_1, Y_1)}
\end{equation*}
For illustration, consider the case where $\theta = 0$. We now take the first measurement at $X_{1} = \phi$ (note $\phi = 1/m$ for PQS). Then
\begin{equation*}
    \bP\left(X_{1} = \phi, Y_{1} = 0 | \theta = 0\right) = 1 - p
\end{equation*}
and
\begin{equation*}
    \bP\left(X_{1} = \phi, Y_{1} = 1 | \theta = 0\right) = p.
\end{equation*}
In fact, this holds for any $\theta < \phi$.
Now examine the denominator:
\begin{eqnarray*}
    \bP(X_{1} = \phi, Y_{1} = 0) &=& \phi(1-p) + (1-\phi)p \\
    &:=& \phi * p,
\end{eqnarray*}
We then update the posterior distribution to be
\begin{equation*}
    \pi_{1}(x) = \begin{cases}
        \frac{(1-p)}{\phi * p} & x \leq \phi \\
        \frac{p}{\phi * p} & x > \phi.
    \end{cases}
\end{equation*}
The equivalent posterior density can be found for when $Y_1=1$.
The process of making an observation and updating the prior is then repeated, yielding general formula for the posterior update. When $Y_{n+1} = 0$, we have
\begin{equation*}
    \pi_{n+1}(x) = \begin{cases}
        \frac{(1-p)}{\phi * p} \pi_{n}(x) & x \leq X_{n+1} \\
        \frac{p}{\phi * p} \pi_{n}(x) & x > X_{n+1}.
    \end{cases}
\end{equation*}
Similarly, for $Y_{n+1} = 1$, we have
\begin{equation*}
    \pi_{n+1}(x) = \begin{cases}
        \frac{p}{\phi * p} \pi_{n}(x) & x \leq X_{n+1} \\
        \frac{(1-p)}{\phi * p} \pi_{n}(x) & x > X_{n+1}.
    \end{cases}
\end{equation*}

\subsubsection{Convergence of Estimation Error}

Analysis of the above algorithm has proven difficult since its inception in 1974, with a first proof of a geometric rate of convergence appearing only recently in \cite{waeber2013bisection}. Instead, the authors and those following use a discretized version involving minor modifications. We follow this strategy, with the discretized algorithm given in the supplemental material \cite{lipor2016supplemental}. In this case, the unit interval is divided into bins of size $\Delta$, such that $\Delta^{-1} \in \bN$. The posterior distribution is
parameterized, and a parameter $\alpha$ is used instead of $p$ in the Bayesian update, where $0 < p < \alpha$.
The analysis of rate of
convergence then centers around the increasing probability that at least half of the mass of $\pi_{n}(x)$ lies in the correct bin. A formal description of the algorithm can be found in the supplemental material \cite{lipor2016supplemental}. Given this discretized version of PQS, we arrive at the following result.

\begin{thm}
    \label{thm:pqs_rate}
    Under the assumptions given in Section~\ref{sec:problem}, the discretized PQS algorithm satisfies
    \begin{equation}
        \sup_{\theta \in [0,1]} \bE[|\tnhat - \theta|] \leq 2 \left( \frac{m-1}{m} + \frac{2\sqrt{p(1-p)}}{m} \right)^{n/2}.
        \label{eq:thm_exp}
    \end{equation}
\end{thm}

The proof can be found in the supplemental material \cite{lipor2016supplemental}. In the case where $m = 2$, the above result matches that of \cite{burnashev1974interval,castro2008active} as desired. One important fact to note is that in contrast to the deterministic case, the result here is an upper bound on the number of samples required for convergence as opposed to an expected value. As this seems to be the case for all analyses of similar algorithms \cite{burnashev1974interval,castro2008active,waeber2013bisection}, we instead rely on Monte Carlo simulations to choose the optimal value of $m$. Finally, the bound here is loose. For clarity, consider the case where $p = 0$ and $m = 2$. Then the above becomes
\begin{equation*}
    \sup_{\theta \in [0,1]} \bE[|\tnhat - \theta|] \leq 2 \left( \frac{1}{2}\right)^{n/2}.
\end{equation*}
As noted in \cite{castro2008minimax}, we can see by inspection that
\begin{equation*}
    \sup_{\theta \in [0,1]} \bE[|\tnhat - \theta|] \leq \left( \frac{1}{2}\right)^{n+1},
\end{equation*}
indicating that we lose a factor of about $n/2$, even for the PBS algorithm bound in \cite{castro2008minimax}. However, in
\cite{castro2008minimax}, the authors use this result when $m = 2$ to show rate optimality of the PBS algorithm. This fact suggests that despite the discrepancy, the result of Thm~\ref{thm:pqs_rate} may still be useful in proving some sort of optimality for the PQS algorithm.

While the rate of convergence for PQS can be derived using standard techniques, the expected distance or a useful bound on the distance is more difficult. The technique used in Section~\ref{sec:DQS} becomes intractable as the values of $X_{n}$ are no longer deterministic. The approach of examining the posterior distribution after each step and calculating the possible locations has been examined, but at the $n$th measurement, there are $2^{n-1}$ possible distributions. Further, PQS as
described above has the
undesirable property that it does not always travel toward the median of the distribution\textemdash a problem we overcome in the next section\textemdash and hence the analysis of its distance is not of any practical importance.

\subsubsection{Truncated PQS}

Probabilistic quantile search as presented in Algorithm \ref{alg:PQS} is not a strict generalization of DQS in the sense that the two algorithms are not equivalent in the noiseless case. Moreover, 
in some cases, PQS will choose a sample location farther away from the current location than the median. This choice is suboptimal, as the median of the posterior is the most informative point (in an information-theoretic sense), and hence traveling a farther distance to obtain less information is contrary to our overall goal. 
For these reasons, we propose the following variant of PQS, which has a sample complexity and distance traveled no worse than the PQS algorithm in Algorithm \ref{alg:PQS}. The algorithm satisfies the statement of Thm. \ref{thm:pqs_rate} (see \cite{lipor2016supplemental}), and we show the improved performance in terms of both distance and sample complexity in Section \ref{sec:simulations}. Instead of taking a sample at the $m$-quantile of the posterior, we instead truncate the posterior distribution in such a way to maintain the median as well as guarantee that the $m$-quantile of this truncated posterior is moving our sampling location towards the median of the posterior (the most informative point).

The generalized algorithm begins by sampling at the $m$-quantile as in PQS. For subsequent samples, we first define
\begin{equation*}
    \chi = \min \set{\int_{0}^{X_{n}} \pi_{n}(x) dx, \int_{X_{n}}^{1} \pi_{n}(x) dx},
\end{equation*}
the probability in the tail of the distribution that would possibly cause us to move away from the median point. 
We then define the truncated distribution to be the normalized form of
\begin{equation*}
    \tilde{\pi}_{n}(x) = \begin{cases}
        0, & \int_{0}^{x} \pi_{n}(z) dz \leq \chi \\
        0, & \int_{x}^{1} \pi_{n}(z) dz \leq \chi \\
        \pi_{n}(x), & \text{otherwise}
    \end{cases}.
\end{equation*}
Finally, the sample location is chosen as $$X_{n+1} = \argmin_{X \in \set{\tilde{X}_{0},\tilde{X}_{1}}} \left| X_{n} - X \right|\;,$$
where
\begin{eqnarray*}
    \int_{0}^{\tilde{X}_{0}} \tilde{\pi}(x) dx = \frac{1}{m} & \text{and} & \int_{\tilde{X}_{1}}^{1} \tilde{\pi}(x) dx = \frac{m-1}{m}.
\end{eqnarray*}
Analogous to traveling ``forward'' or ``backward'' in DQS, this process guarantees that we always choose sample locations that are in the direction of the median of the posterior. This fact ensures that the information gain is at least that of the PQS algorithm, while choosing the nearer of the two locations results in a distance no greater than that of PQS. Note that we continue to use $\pi_{n}(x)$ as the posterior distribution of $\theta$ and update this distribution according to Algorithm
\ref{alg:PQS}, \ie we only use $\tilde{\pi}_{n}(x)$ when choosing the sample locations. Further, for the case of $m = 2$, this generalization and PQS are equivalent, both resulting in the PBS algorithm.

\subsubsection{Stopping Criterion}
\label{sec:stopping}

Previous work on PBS centers around the case where there is a fixed sample budget, avoiding the need for a stopping criterion for this algorithm. However in our application, while we need to reduce sampling resources, we only stop sampling once we have reached a desired accuracy. In this case, one natural choice of stopping criteria for PBS would be to stop when the distance between successive samples is smaller than some predetermined value. However, in the case of PQS with high $m$, the step size may be very small from the start, resulting in early termination. In the case of DQS, the width of the feasible interval provides a direct measure of the 
absolute error in estimating $\theta$. While there is no such width in the case of PQS, the certainty in our estimate of $\theta$ is quantified via the posterior distribution $\pi_{n}(x)$, which is discretized in our implementation. In light of this, we terminate PQS (or its generalized version) when there exists an $x_{i}$ such that $\pi_{n}(x_{i}) \geq 0.9$.

\subsection{Algorithmic Improvements}
\label{sec:improvements}

\begin{figure}[t]
    \centering
    \includegraphics[width=0.7\linewidth]{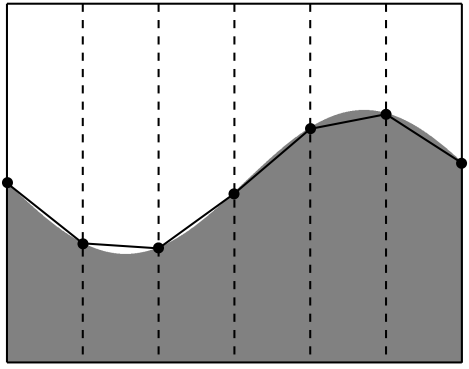}
    \caption{Example of set belonging to boundary fragment class and piecewise linear estimation of boundary.}
    \label{fig:holder}
\end{figure}

\begin{figure*}[h!]
    \centering
    \includegraphics[width=\linewidth]{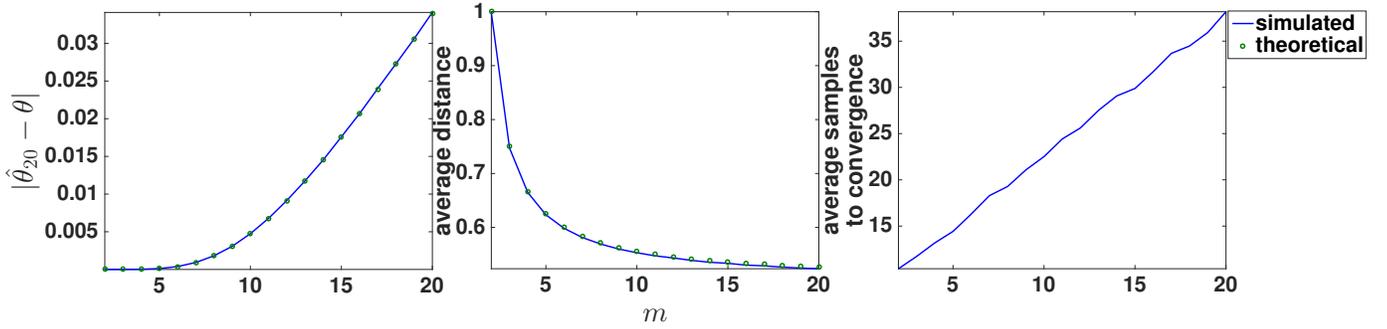}
    \caption{Simulated and theoretical values for DQS. Left-to-right: expected error after 20 samples, distance traveled before convergence to an estimation error less than $1 \times 10^{-4}$, simulated average samples required to converge to the same error.}
    \label{fig:dqs}
\end{figure*}
In this section, we describe two heuristics that can be used to reduce the sampling time. These heuristics are appropriate in the case where the decision boundary is smooth in some sense and is estimated using a series of successive quantile searches. Consider the boundary fragment class on $[0,1]^{d}$ defined informally in
\cite{castro2008minimax} as the collection of sets in which the Bayes decision boundary is a H{\"o}lder smooth function of the first $d-1$ coordinates. In $[0,1]^{2}$, this implies that the boundary is crossed at most one time when traveling on a path along the second coordinate. The boundary can be estimated by dividing the problem into strips along the first dimension, estimating the change point of each strip, and estimating the boundary as a piecewise linear function of the
estimates, as shown in Fig. \ref{fig:holder}. For simplicity, we motivate the heuristics in this section by restricting $f$ to the class of Lipschitz functions (a subset of H{\"o}lder smooth functions). Recall that a function $f: [0,1]^{d} \rightarrow
\bR$ is said to be Lipschitz with constant $L \geq 0$
if for all $x_{1} \neq x_{2}$
\begin{equation*}
    \left| f(x_{1}) - f(x_{2}) \right| \leq L \norm{x_{1} - x_{2}}.
\end{equation*}
Returning to Fig. \ref{fig:holder}, we see that a great deal of time would be wasted by returning to the origin after estimating the boundary at each strip. In this section, we leverage the assumed smoothness to intelligently initialize quantile search, resulting in significantly reduced sampling times, as shown in the simulations. 

\subsubsection{Initialization Using Previous Estimate}

Assume we split the region of interest into $K$ strips, each of which is a step function on the unit interval whose change point we wish to estimate. Let the true change point of the $k$th strip be $\theta^{k}$ and the estimate be $\hat{\theta}^{k}$. The smoothness assumption implies that $\theta^{k+1}$ is not ``too far'' from $\theta^{k}$. For example, if $f$ is Lipschitz with constant $L$ and two successive strips are located at $x_{k}$ and $x_{k+1}$, we know that $| \theta^{k} - \theta^{k+1} | / |
x_{k} - x_{k+1} | \leq L$. For this reason, our first proposed improvement is to let the first sample location of the $k+1$st strip $X_{0}$ be the previous estimate $\hat{\theta}^{k}$. Note that if we further assume a uniform prior on the subinterval allowed by the smoothness assumption, we are choosing our first sample as the minimum absolute error estimate, \ie the median of the distribution. For later reference, we refer to this initialization as Improvement 1 (I-1). We show in Section \ref{sec:simulations} that this simple heuristic dramatically reduces the required sampling time of our algorithm.

\subsubsection{Nonuniform Priors}

Our second proposed algorithmic improvement involves assigning a nonuniform prior when beginning the search. Similar to the previous improvement, we utilize the function smoothness to assign lower starting probabilities to points unlikely to lie near the decision boundary. Letting $\hat{\theta}^{k}$ again be the boundary estimate at the $k$th strip, we assign a nonuniform prior whose mean is centered around $\hat{\theta}^{k}$. We propose the use of either a
piecewise uniform or a Gaussian kernel function and refer to these as I-2.1 and I-2.2, respectively. Let the strip width $|x_{k}-x_{k+1}| = W$. We assign the prior probability for the $k+1$st strip to be either
\begin{equation*}
    \pi_{0}(x) = \begin{cases}
        c_{1}, & \left| x - \hat{\theta}^{k} \right| \leq LW \\
        c_{2}, & \left| x - \hat{\theta}^{k} \right| > LW
    \end{cases} \qquad (\mathrm{I-2.1}),
\end{equation*}
where $c_{1} > c_{2}$, or
\begin{equation*}
    \pi_{0}(x) = c_{3}\exp\left(-\frac{(x-\hat{\theta}^{k})^{2}}{2(LW)^{2}}\right) \qquad (\mathrm{I-2.2}),
\end{equation*}
where $c_{3}$ is a normalization constant so that the prior sums to 1.
We discuss the choice of $L$ and $W$ in Section \ref{sec:simulations}.

\section{Simulations \& Experiments}
\label{sec:simulations}

In this section, we show the efficacy of our algorithm through simulations. We first verify the theoretical guarantees provided in Section \ref{sec:QS} and then compare the performance of PQS with the generalized version, which we refer to as TPQS. Next, we compare our method to proactive learning from \cite{donmez2008proactive}. We then show how a series of one-dimensional searches can be used to estimate the boundary of a two-dimensional hypoxic region in Lake Erie. We conclude with experimental results from Third Sister Lake in Ann Arbor, MI.

\begin{figure*}
    \centering
    \includegraphics[width=0.7\linewidth]{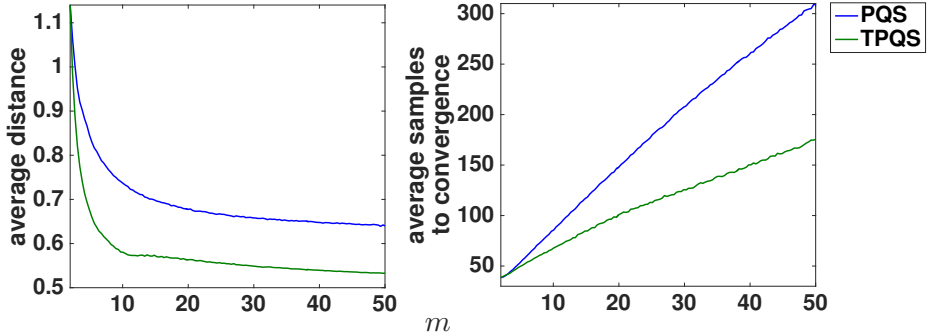}
    \caption{Average simulated values for PQS and TPQS. Left-to-right: distance traveled during estimation and number of samples required to converge.}
    \label{fig:pqs}
\end{figure*}
\subsection{Verification of Algorithms}

In this section, we verify through simulation the theoretical rate of convergence and distance traveled derived in Section~\ref{sec:DQS}. Further, we present simulated results for the PQS and TPQS algorithms and show that the desired tradeoff is achieved by both algorithms, with TPQS achieving better overall performance.

We first simulate the the DQS algorithm over a range of $m$ from 2 to 20, where $\theta$ is swept over a 1000-point grid on the unit interval. The resulting average error after 20 samples is shown in the left-most plot of Fig.~\ref{fig:dqs}, while the average distance before convergence to an error of $\varepsilon = 1 \times 10^{-4}$ is shown in the middle plot of the same figure. The figures show the theoretical values for expected error and distance match the simulated values. The
right-most plot of Fig. \ref{fig:dqs} shows the number
of samples required to converge to the same error. From the figures, our intuition is confirmed; the number of samples required is monotonically increasing in $m$, while the distance traveled is monotonically decreasing. This indicates that DQS achieves the desired tradeoff in the noise-free case. 

Next, we simulate the PQS and TPQS algorithms with error probability $p = 0.1$ over a range of $m$ from 2 to 50, where $\theta$ ranges over a 100-point grid on the unit interval with 100 random instances run for each value of $\theta$. The left-hand plot of Fig.~\ref{fig:pqs} shows the average number of samples required to converge to a mass of at least 0.9 at a single point, as described in Section \ref{sec:stopping}. As in the deterministic case, the required number of samples increases monotonically with
$m$. The right-hand plot of Fig.~\ref{fig:pqs} shows the average distance traveled before converging to the same error value. Again, the distance decreases monotonically with $m$, indicating that the algorithm achieves the desired tradeoff in the noisy case. Further, we see that TPQS outperforms PQS both in terms of samples required and distance traveled. Because of this, we consider only TPQS in all remaining simulations.

\begin{figure*}[h!]
    \centering
    \includegraphics[width=\linewidth]{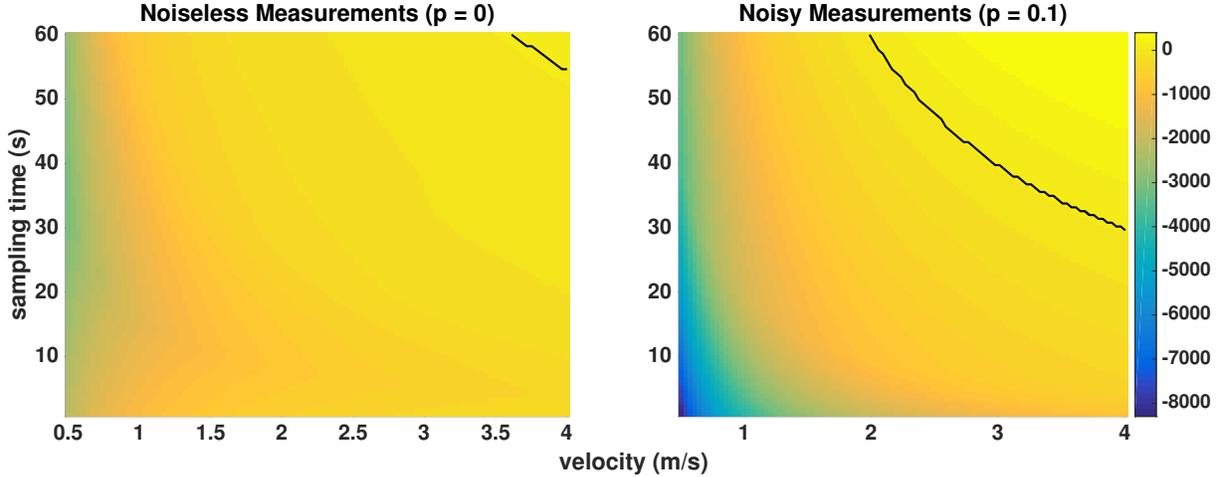}
    \caption{Difference in sampling time between quantile search and proactive learning under a variety of practical sampling regimes for both noiseless (left) and noisy (right) measurements with $p = 0.1$. Quantile search results in less required time for all points ``southwest'' of the black line.}
    \label{fig:comp}
\end{figure*}

\subsection{Application of Proactive Learning}
\label{sec:proactive}

The most competitive algorithm to quantile search is that of \cite{donmez2008proactive} applied to our problem. Of the scenarios explored in \cite{donmez2008proactive}, the most relevant is Scenario 3, in which a non-uniform cost is charged for each label. The authors propose choosing each sample location to maximize the utility $U(X)$ at each round, where utility is defined as the difference between the value of the sample at $X$ and the cost of taking that sample. The authors alternatively define utility as the ratio of value to cost, lending to a more natural interpretation that is similar to
\cite{rahimi2005adaptive}. However, we found this version to result in poor performance, and hence we rely on the first approach. To compare with quantile search, we maintain an estimate of the posterior distribution of $\theta$ as in quantile search. We define the value of a point $X$ as the mutual information $I(\theta;X,Y)$ \cite{cover2012elements}. Note that in the noiseless case, $Y$ is a deterministic function of $\theta$, and hence mutual information is a misnomer. In this case, we
still consider the reduction in entropy of $\theta$ given the measurement $Y$ taken at point $X$. The relation of binary search to communicating a noisy sequence of bits over a binary symmetric channel has been well-studied
\cite{horstein1963sequential}. In the
noiseless case, we have
\begin{equation*}
    H(\theta) - H(\theta | X,Y) = H_{b}(X),
\end{equation*}
where $H(\cdot)$ denotes the differential entropy and $H_{b}(\cdot)$ is the entropy of a Bernoulli random variable with corresponding probability $X$. 

The noisy case of Section \ref{sec:PQS} corresponds to a binary symmetric channel with non-uniform priors \cite{jiang2006multilevel}. In this case, we have
\begin{eqnarray*}
    I(\theta;X,Y) &=& H(\theta) - H(\theta | X,Y) \\
    &=& H_{b}(\phi * p) - H_{b}(p),
\end{eqnarray*}
where
\begin{equation*}
    \phi = \int_{0}^{X} \pi(x) dx.
\end{equation*}
Note that for $\phi = 1/2$ (\ie PBS), the mutual information is $1 - H_{b}(p)$, which is the capacity of a noisy binary symmetric channel. We implement the proactive algorithm from \cite{donmez2008proactive}, Eqn. (5) with two modifications in order to provide a fair comparison. First, the non-uniform cost in our case is the distance between the current location and the point under consideration, rather than the generic cost described in \cite{donmez2008proactive}. Second, we provide a tuning parameter that can be used similarly to $m$ to balance between the number of samples and distance traveled during estimation.
Pseudocode for this algorithm is given in Algorithm \ref{alg:proal}. In both the noiseless and noisy cases, we use the stopping criteria from DQS and PQS, respectively.

\begin{algorithm}[t]
    \caption{Proactive Learning with Non-Uniform Costs}
    \label{alg:proal}
    \begin{algorithmic}[1]
        \STATE \textbf{Input:} search parameter $\lambda \in [0,1]$, probability of error $p$
        \STATE \textbf{Initialize:} $\pi_{0}(x) = 1$ for all $x \in [0,1]$, $n \gets 0$
        \WHILE{not converged}
        \STATE $X_{n+1} = \argmax_{x} I(\theta;x,Y) - \lambda |X_{n} - x|$
        \STATE $Y_{n+1} \gets f(X_{n+1}) \oplus U_{n+1}$, where $U_{n+1} \sim$ Ber($p$)
        \STATE $\phi = \int_{0}^{X_{n+1}} \pi_{n}(x) dx$
        \IF{$Y_{n+1} = 0$}
        \STATE \begin{equation*}
            \pi_{n+1}(x) = \begin{cases}
                \frac{(1-p)}{\phi * p} \pi_{n}(x) & x \leq X_{n+1} \\
                \frac{p}{\phi * p} \pi_{n}(x) & x > X_{n+1}.
            \end{cases}
            \end{equation*}
        \ELSE
        \STATE \begin{equation*}
            \pi_{n+1}(x) = \begin{cases}
                \frac{p}{\phi * p} \pi_{n}(x) & x \leq X_{n+1} \\
                \frac{(1-p)}{\phi * p} \pi_{n}(x) & x > X_{n+1}.
                \end{cases}
            \end{equation*}
        \ENDIF
        \STATE $n \gets n + 1$
        \ENDWHILE
        \STATE estimate $\tnhat$ such that $\int_{0}^{\tnhat} \pi_{n}(x) = 1/2$
    \end{algorithmic}
\end{algorithm}

To obtain a profile of the performance of proactive learning, we simulate for both noiseless and noisy ($p = 0.1$) measurements, where we range $\lambda$ over 100 points on the unit interval. We let $\theta$ range over a 100-point grid, and 200 random instances are run for each value of $\theta$ in the noisy case. To compare with DQS and TPQS, we simulate the average time required to perform sampling on the unit interval under a variety of sampling times and travel times relevant to our
problem of sampling in Lake Erie. We let the time per sample $\eta$ in \eqref{eq:time} range from 1-60 s per sample. For travel time, we consider a strip length of 40 km, about half the size of the central basin of Lake Erie, and let the velocity range from 0.5-4 m/s. Fig. \ref{fig:comp} shows the difference in sampling time required by quantile search and proactive learning. The boundary of where quantile search outperforms proactive learning is shown in black, so that all points ``up
and to the right'' of the boundary denote sampling regimes in which proactive learning requires less time than quantile search. The figure shows that in the majority of relevant cases, quantile search results in superior performance. However, in the case of large sampling time and high velocity, proactive learning generally performs better. Although not shown, we analyzed figures similar to Figs. \ref{fig:dqs} and \ref{fig:pqs} and saw that the number of samples required for proactive
learning to converge reduces quickly with
$\lambda$ compared to quantile search, while the distance traveled reduces slowly. Thus, for scenarios in which sampling is significantly more costly than travel, proactive learning may be a more appropriate choice. This is likely due to the fact that proactive learning often takes comparatively large steps early in the measurement process, and investigating the properties of this algorithm is a topic for our future research.

\begin{figure}[h!]
    \begin{subfigure}[b]{\linewidth}
        \includegraphics[width=\linewidth]{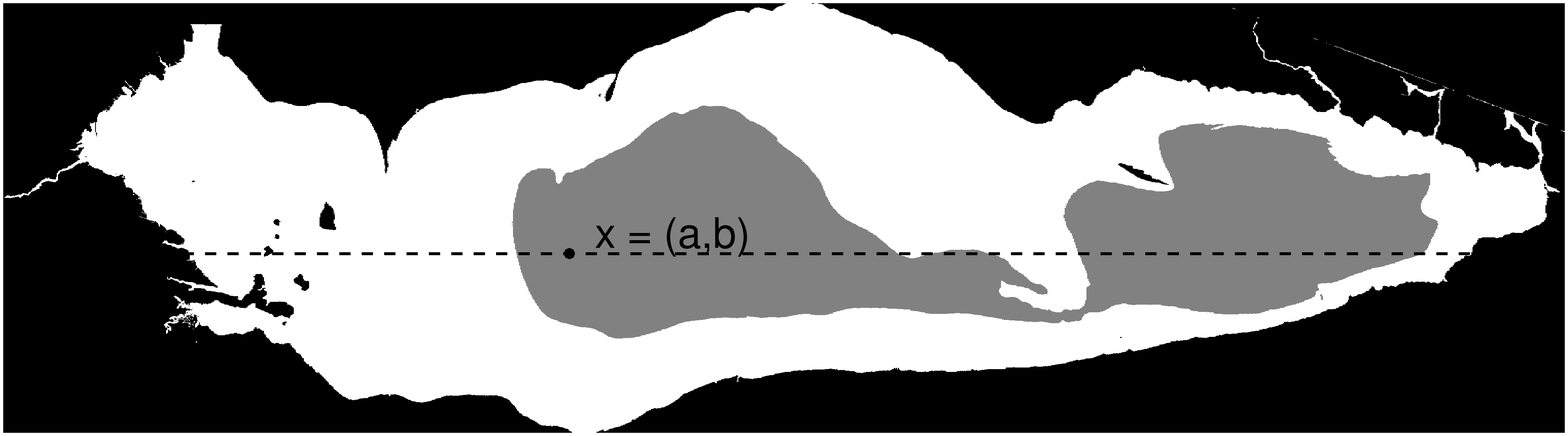}
        \caption{}
        \label{fig:erie1}
    \end{subfigure}
    \begin{subfigure}[b]{\linewidth}
        \includegraphics[width=\linewidth]{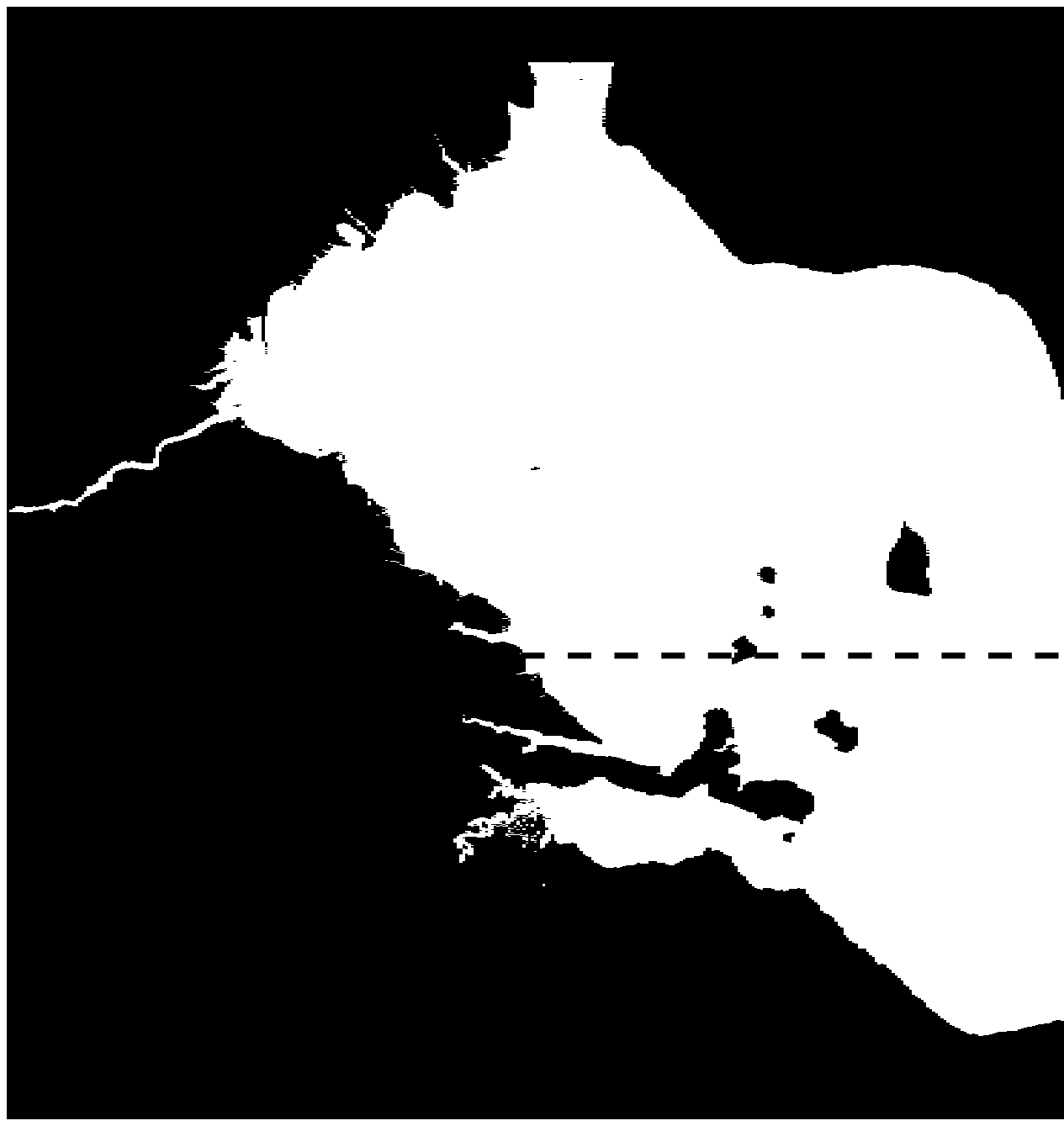}
        \caption{}
        \label{fig:erie2}
    \end{subfigure}
    \begin{subfigure}[b]{\linewidth}
        \includegraphics[width=\linewidth]{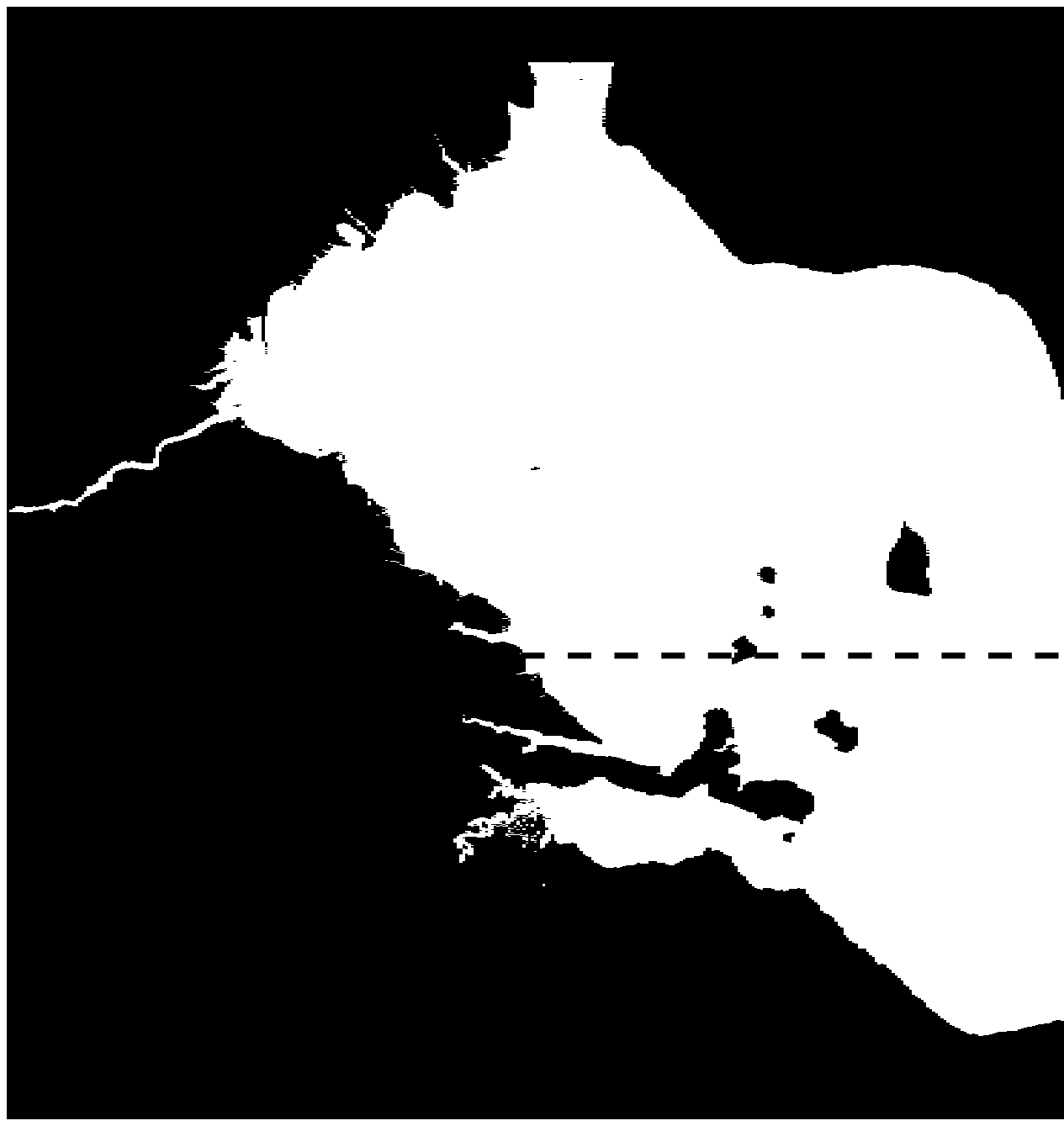}
        \caption{}
        \label{fig:erie3}
    \end{subfigure}
    \begin{subfigure}[b]{\linewidth}
        \includegraphics[width=\linewidth]{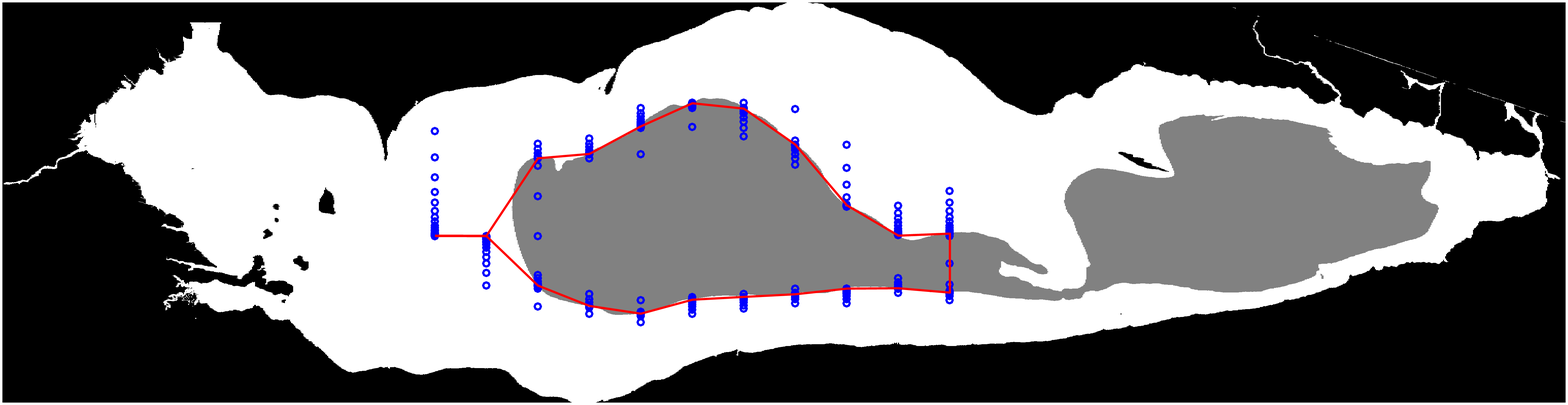}
        \caption{}
        \label{fig:erie4}
    \end{subfigure}
    \caption{Proposed sampling procedure for detection of hypoxic region in Lake Erie. (a) Lake Erie with hypoxic region illustrated in gray and split along $x = (a,b)$. (b) Division of top portion into strips. (c) Estimation procedure for top of lake with sample locations shown in blue and estimated boundary in solid red. (d) Final sample locations and estimation of entire boundary.}
\end{figure}

\subsection{Simulations on Lake Erie}

\begin{table*}[ht!]
    \centering
    \begin{tabular}{ | c | c || c || c || c || c | }
        \hline
        \multirow{2}{*}{Sampling Scenarios} & Sampling Time (s) & 60 & 60 & 10 & 10 \\
        & Velocity (m/s) & 4 & 0.5 & 4 & 0.5 \\
        \hline
        \multirow{2}{*}{Sampling Parameters} & $m$ & 9.48 & 43.00 & 43.00 & 43.00 \\
        & $\lambda$ & 0.17 & 0.13 & 0.13 & 0.10 \\
        \hline \hline
        \multirow{3}{*}{No Improvement} & Bisection & 2.14 & 15.96 & 2.00 & 15.82 \\
        & DQS & 2.62 & 19.97 & 2.52 & 19.44 \\
        & Proactive Learning & 2.84 & 21.20 & 2.67 & 20.45 \\
        \hline
        \multirow{3}{*}{I-1} & Bisection & 1.99 & 14.86 & 1.86 & 14.73 \\
        & DQS & \textbf{1.44} & \textbf{9.41} & \textbf{1.19} & \textbf{9.09} \\
        & Proactive Learning & 1.47 & 9.99 & 1.26 & 9.62 \\
        \hline
    \end{tabular}
    \caption{Total sampling time (in days) for various search methods under noiseless measurements and a variety of sampling times and velocities. Fastest time for each scenario shown in bold.}
    \label{tab:dqs}
\end{table*}

In this section, we apply the quantile search and proactive learning algorithms to the problem of sampling hypoxic regions in Lake Erie. Fig.~\ref{fig:erie1} shows the lake with an example hypoxic zone pictured in gray. In \cite{castro2008minimax}, the authors show that for the set of distributions such that the Bayes decision set is a boundary fragment, a variation on PBS can be used to estimate the boundary while achieving optimal rates up to a logarithmic factor. We now describe how the same approach can be
used to estimate the hypoxic region in Lake Erie and demonstrate the benefits of our algorithm compared to PBS and proactive learning. The results in this section differ from our previous work \cite{lipor2015quantile} in that we consider a more realistic boundary derived from bathymetry data retrieved from \cite{noaa2015bathymetry}. To simulate the boundary of interest, we threshold the bathymetry data at a depth of 21 m and consider anything at a depth of greater than 21 m hypoxic. Although
this may not be directly correlated with the hypoxic region, the resulting region is sufficiently irregular to test our algorithm and is visually similar to the regions found in \cite{zhou2013spatial}. Further, we previously considered only the time required to estimate the strips (described below) individually, whereas in this work we consider the entire sampling process.

Consider the instance of a hypoxic region shown in Fig.~\ref{fig:erie1}. 
Using models and measurements from previous years (\eg historical data from \cite{zhou2013spatial}), it is reasonable to assume we can split the lake into intervals so that the boundary does not significantly violate the boundary fragment assumption. Splitting the lake along the line $y = b$ yields the two sets above and below the dashed line in Fig.~\ref{fig:erie1}. Now we can further divide the problem into strips along the first dimension, as shown by the solid red line in Fig.~\ref{fig:erie2}. Along each of these strips, the problem reduces to change point
estimation of a one-dimensional threshold classifier as we have studied thus far. After estimating the change point at each strip, the boundary is estimated as a piecewise linear function of the estimates, as shown in Fig.~\ref{fig:erie3}. The same procedure is used for the bottom portion of the lake, with the final estimation shown in Fig.~\ref{fig:erie4}. In all cases, we choose the optimal $m$ by estimating the average number of samples and distance traveled via simulations and note the
chosen value in the tables.

\begin{table*}[t]
    \centering
    \begin{tabular}{ | c | c || c || c || c || c | }
        \hline
        \multirow{2}{*}{Sampling Scenarios} & Sampling Time (s) & 60 & 60 & 10 & 10 \\
        & Velocity (m/s) & 4 & 0.5 & 4 & 0.5 \\
        \hline
        \multirow{2}{*}{Sampling Parameters} & $m$ & 6.40 & 11.17 & 10.80 & 62.16 \\
        & $\lambda$ & 0.30 & 0.29 & 0.29 & 0.29 \\
        \hline \hline
        \multirow{3}{*}{I-1} & PBS & 2.64 & 19.00 & 2.38 & 18.61 \\
        & TPQS & 1.83 & 10.84 & 1.38 & 9.57 \\
        & Proactive Learning & \textbf{1.72} & 11.67 & 1.48 & 11.47 \\
        \hline
        \multirow{3}{*}{I-1, I-2.1} & PBS & 2.63 & 18.66 & 2.37 & 18.66 \\
        & TPQS & 1.82 & 10.85 & 1.38 & 9.58 \\
        & Proactive Learning & 1.73 & 11.69 & 1.47 & 11.44 \\
        \hline
        \multirow{3}{*}{I-1, I-2.2} & PBS & 2.58 & 18.30 & 2.33 & 18.11 \\
        & TPQS & 1.83 & \textbf{10.75} & \textbf{1.37} & \textbf{9.56} \\
        & Proactive Learning & 1.73 & 11.77 & 1.49 & 11.52 \\
        \hline
    \end{tabular}
    \caption{Total sampling time (in days) for various search methods under noisy measurements with $p = 0.1$ and a variety of sampling times and velocities. Fastest time for each scenario shown in bold.}
    \label{tab:pqs}
\end{table*}

We apply this procedure to the hypoxic region shown in Fig.~\ref{fig:erie1} using 11 strips for a variety of values for time per sample and speed of watercraft. To simulate an actual sampling pattern, we proceed counterclockwise through the strips, beginning from the top left, and record the total distance traveled and number of samples taken. We consider several sampling strategies. As a baseline, we use binary bisection with no algorithmic improvements, \ie quantile search with fixed
$m = 2$. We also show DQS with a fixed $m$ chosen to optimize the total sampling time using the average scale factor for the entire lake. Next, we show the sampling time for proactive learning with $\lambda$ chosen similarly. Finally, we consider these scenarios while employing Improvement 1, where we initialize our search algorithm using the previous boundary estimate. We forego the application of I-2.1 and I-2.2, as they will have minimal impact in the noiseless case. Table~\ref{tab:dqs} shows the resulting sampling time (in days) required to estimate the boundary of the hypoxic region. When I-1 is not in use, binary
search outperforms our algorithm. This is due to the fact that the craft must travel back to the position $1/m$ at each strip, a significant distance when $m$ is small and the boundary estimate is not near the origin. However, this problem is overcome by employing I-1, in which case quantile search requires roughly half the sampling time required by binary search. Further, DQS outperforms proactive learning, even in the scenarios where DQS requires more time on a single strip. In the case of low sampling time and low velocity, we see that DQS significantly outperforms proactive learning, which matches our expectations based on Fig. \ref{fig:comp}.

Next, we apply the same procedure to the noisy case with a probability of measurement error $p = 0.1$ averaged over 100 random instances. Due to the performance benefits shown in DQS, we employ I-1 in all sampling scenarios. For both I-2.1 and I-2.2, we choose the strip width $W$ based on the number of strips, which is a function of the desired estimation error. Choosing $W$ small will result in more accurate estimation but require more sampling time. In practice one would estimate $L$ using historical data. We estimate $L$ numerically as
\begin{equation*}
    \hat{L} = \argmax_{i} \frac{\left| f(x_{i}) - f(x_{i}+\delta) \right|}{\delta},
\end{equation*}
where we choose $\delta$ to be $0.1W$ to prevent the value of $L$ from being inflated by a single point in $f$ with high derivative. Note that since we are only using $L$ to generate priors for our search function, even an aggressive choice will not prevent our algorithm from finding the true boundary. In some cases, where the function is very smooth in many places and has high derivative in a few places, a user may wish to choose $L$ smaller than the estimated value to reduce
sampling time. In I-2.1, we choose $c_{1}$ and $c_{2}$ such that the probability within $LW$ of $\hat{\theta^{k}}$ is 100 times the probability outside this region, \ie $c_{1} = 100c_{2}$.

Table \ref{tab:pqs} shows the resulting sampling times under the various sampling scenarios. The results of the table across all sampling parameters indicate that TPQS with a Gaussian prior is the best sampling strategy in most cases. In the case of a 60 sec sampling time and 4 m/s velocity, proactive learning outperforms our algorithm, which is consistent with the results of Fig. \ref{fig:comp}. Interestingly, the use of nonuniform priors results in a small benefit in most cases. This is
likely due to the fact that the bottom half of the hypoxic region is extremely smooth, and hence our value of $L$ is not aggressive enough for these strips. A better choice may be to choose $L$ separately for the strips on top and bottom of the lake.

\subsection{Experiments on Third Sister Lake}

\begin{figure}[t]
    \centering
    \includegraphics[width=\linewidth]{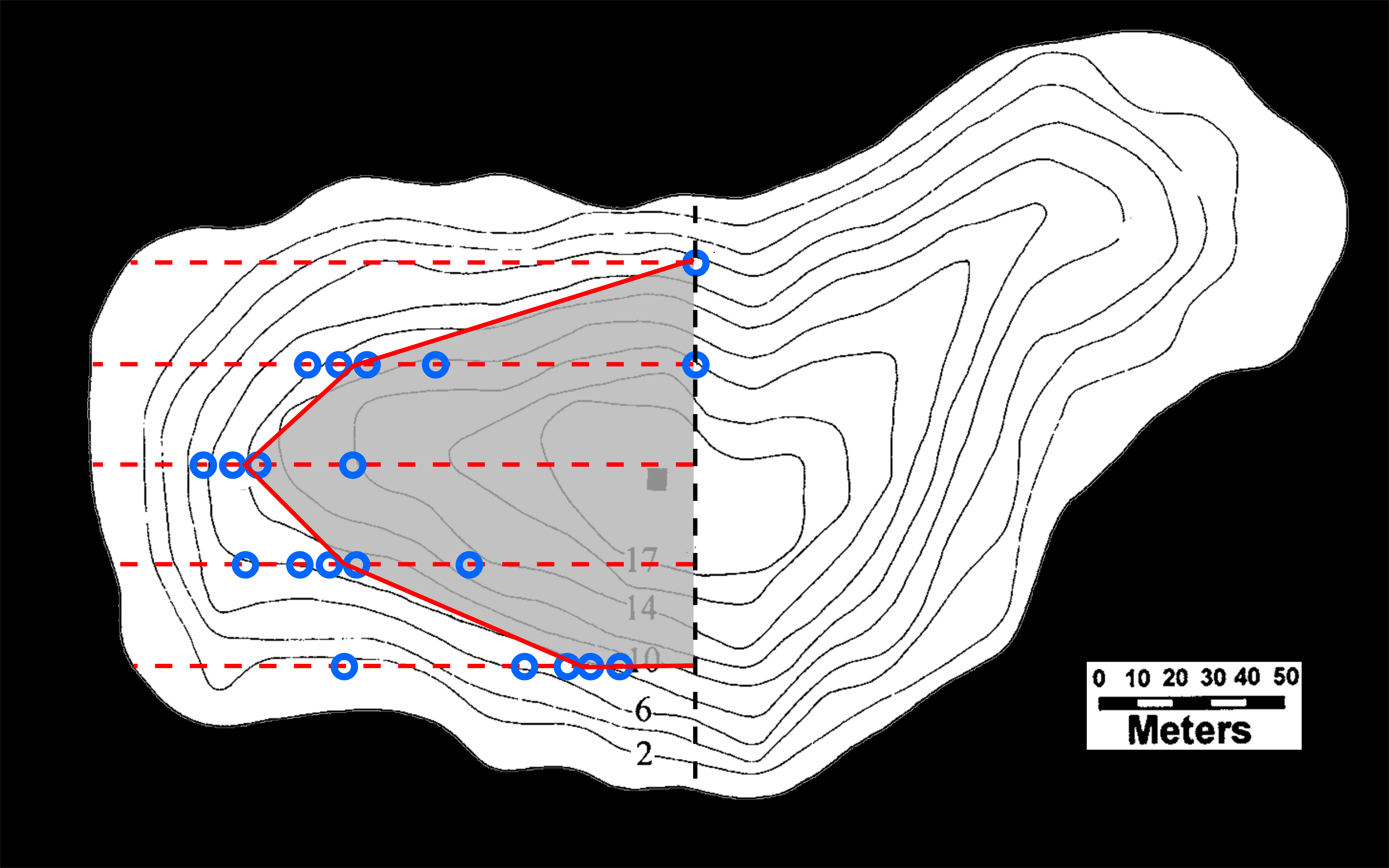}
    \caption{Delineated hypoxic region on the western half of Third Sister Lake.}
    \label{fig:experiment}
\end{figure}

In this section, we present the implementation and performance of the DQS algorithm in the field. The algorithm was tested on a robotic boat that was deployed at Third Sister Lake in Ann Arbor, Michigan. Third Sister Lake is a spring-fed kettle lake with an area of 9.4 acres and a maximum depth of 17 meters that notably exhibits hypoxic conditions on an annual basis \cite{bridgeman2000limnological}. The smaller size and calmer waters of Third Sister Lake posed an ideal test bed for evaluating the algorithm.

The robotic boat platform \cite{valada2014development} features an Android cellular phone for GPS navigation and 3G cellular communications. The prevalent cell coverage at Third Sister Lake enables bi-directional communication with the boat for remotely tracking and delineating the evolution of the hypoxic region in real-time. For a given GPS coordinate, the boat autonomously navigates to the destination to collect a sample. The platform was outfitted with a motorized winch to raise and lower a suite of water
quality sensors to measure dissolved oxygen throughout the water column at each sampling location. Due to the low noise level of these sensors, we employed the noiseless version of the algorithm with I-1.

Leveraging the persistent Internet connectivity of the robotic boat, the platform was paired with a web-service-based cyberinfrastructure \cite{wong2016real}. This enabled the same script used to develop the algorithm to be tested in the field by modifying the script to open a web connection and directly control the boat. Time-stamped location and measurement data were immediately accessible to the algorithm to direct where the boat should sample next. Taking a web-based approach provides the flexibility
more readily interface with any web-enabled robotic boat that may be more suitable for increased winds and waves of more challenging sites.

We present the results from a sampling campaign on November 17, 2015 in Fig. \ref{fig:experiment}. Third Sister Lake was divided into five horizontal strips, along which an average of five samples were taken until GPS precision could no longer distinguish between two locations. The estimated velocity of the robotic boat was 0.1 m/s. Due to the need to lower and raise the winch for each sample location, the average time to collect a sample was 300 s, resulting in an optimal sampling parameter of $m = 2$. We observed that over the course of five hours, the platform successfully identified and delineated the hypoxic
zone as directed by the algorithm. In comparison, a uniform sampling at the same resolution would take an estimated 27 hours. The successful results from the experiments on Third Sister Lake demonstrate the potential to extend this algorithm to other lake systems including Lake Erie. 

\section{Conclusions \& Future Work}
\label{sec:conclusion}

We have presented an active learning algorithm for spatial sampling capable of balancing the number of samples and distance traveled in order to minimize the overall sampling time. To the best of our knowledge, this is the only nonuniformly penalized active learning algorithm accompanied by theoretical guarantees. We have shown how our algorithm can be used to estimate a two-dimensional region of hypoxia under certain smoothness assumptions on the boundary, and empirical results indicate the
benefits of quantile search over traditional binary search as well as other active learning methods in the literature.

Several open questions remain. Deriving or bounding the expected distance for the GPQS algorithm is an important next step. The boundary fragment class mentioned here is restrictive \cite{castro2008active}, and the extension to more general cases would be of interest. The recent
work of \cite{desarathy2015efficient} describes a graph-based algorithm that employs PBS to higher-dimensional nonparametric estimation. Extending this idea to penalize distance traveled is a promising avenue for practical applications of quantile search. Finally, the PQS algorithm requires knowledge of the noise parameter $p$ in order to update the posterior. The algorithms presented in \cite{wang2016noise,ramadas2013algorithmic} enjoy the property that they are adaptive to unknown noise levels. The development
of a noise-adaptive probabilistic search would certainly be of great interest, with potential applications in areas such as stochastic optimization \cite{ramadas2013algorithmic} beyond direct applicability to this problem.

\bibliographystyle{IEEEtran}
\bibliography{IEEEabrv,Bibliography}

\end{document}